\documentclass{article}

\usepackage{amsmath, graphicx, amsfonts, amsthm}
\usepackage{wrapfig}
\usepackage{framed}
\usepackage{comment}
\usepackage{subcaption}
\usepackage{booktabs}
\usepackage{IEEEtrantools}
\usepackage{mdframed}
 \newmdtheoremenv[
    tikzsetting={rounded corners=5pt},
    backgroundcolor=gray!10,
    linecolor=white,
    innertopmargin=4pt,
    innerbottommargin=4pt,
    skipabove=5pt,
    skipbelow=5pt
]{proposition}{Proposition}
 \newmdtheoremenv[
    tikzsetting={rounded corners=5pt},
    backgroundcolor=gray!10,
    linecolor=white,
    innertopmargin=4pt,
    innerbottommargin=4pt,
    skipabove=5pt,
    skipbelow=5pt
]{theorem}{Theorem}
 \newmdtheoremenv[
    tikzsetting={rounded corners=5pt},
    backgroundcolor=gray!10,
    linecolor=white,
    innertopmargin=4pt,
    innerbottommargin=4pt,
    skipabove=5pt,
    skipbelow=5pt
]{corollary}{Corollary}
\usepackage{MnSymbol}
\usepackage{multirow}
\usepackage[most]{tcolorbox}
\newtcolorbox{eqbox}{
  colback=gray!5,
  colframe=black!40,
  boxrule=0.5pt,
  arc=2pt,
  left=6pt,
  right=6pt,
  top=6pt,
  bottom=6pt
}
\usepackage{amsthm}
\theoremstyle{definition}

\theoremstyle{plain}
\newtheorem{lemma}{Lemma}[section]
\theoremstyle{remark}

 \usepackage[preprint]{neurips_2026}

\usepackage{todonotes}
\usepackage[utf8]{inputenc} % allow utf-8 input
\usepackage[T1]{fontenc}    % use 8-bit T1 fonts
\usepackage{hyperref}       % hyperlinks
\usepackage{url}            % simple URL typesetting
\usepackage{booktabs}       % professional-quality tables
\usepackage{amsfonts}       % blackboard math symbols
\usepackage{nicefrac}       % compact symbols for 1/2, etc.
\usepackage{microtype}      % microtypography
\usepackage{xcolor}         % colors

\title{Geometry-Induced Diffusion on Graphs: A Learnable Weighted Laplacian for Spectral GNNs}

\author{%
  Mia Zosso$^{*1}$ \And
  Ali Hariri$^{*1}$ \And
  Victor Kawasaki-Borruat$^{*1}$\thanks{Correspondence: victor.borruat@epfl.ch} \And
  Pierre-Gabriel Berlureau$^{*2}$ \And
  Pierre Vandergheynst$^{1}$ \\
  \\
  $^{1}$École Polytechnique Fédérale de Lausanne (EPFL), Lausanne, Switzerland \\
  $^{2}$École Normale Supérieure -- PSL, Paris, France \\
}

\begin{document}

\maketitle

\begin{abstract}
Long-range graph tasks are challenging for Graph Neural Networks (GNNs): global mechanisms such as attention or rewiring schemes can be computationally expensive, while deep local propagation is prone to vanishing gradients, oversmoothing, and oversquashing. The introduced $\mu$-ChebNet architecture is a simple spectral GNN that learns a node-wise weight function $\mu$ before applying ChebNet-style filters. The learned weighting $\mu$ induces a modified graph Laplacian which effectively changes the propagation geometry without altering the graph topology. This task-dependent geometry promotes preferred routes for information propagation, thereby helping long-range signals avoid highly contractive bottlenecks, and obviating the need for repeated layer stacking. In practice, we replace the fixed graph Laplacian $L$ by a learned operator $L_\mu$, keeping the proposed $\mu$-ChebNet architecture lightweight while making propagation task-adaptive. Furthermore, we provide a spectral analysis demonstrating how $\mu$ modulates propagation dynamics, and empirically observe improved performance on both synthetic long-range reasoning tasks and real-world graph benchmarks. The learned weight function is not only interpretable, but also offers a lightweight alternative to attention and rewiring for adaptive graph propagation.
\end{abstract}

\section{Introduction}
Graph Neural Networks (GNNs) \citep{scarselli2008graph,bruna2013spectral,defferrard2016convolutional} have achieved significant success in recent years, driving the development of novel architectures and investigations into their theoretical foundations.
Alongside these advances, important limitations have also come to light, most notably oversmoothing and oversquashing \citep{shao2023unifying, rusch2023survey, jin2025oversmoothing, attali2024rewiringtechniquesmitigateoversquashing}.
Oversmoothing makes node representations increasingly indistinguishable, while oversquashing compresses information from many distant nodes into limited-size embeddings.
These effects are particularly damaging for long-range tasks, where relevant signals and gradients must cross many propagation steps and may progressively weaken \citep{arroyo2025vanishing}.

% Graph Neural Networks (GNNs) \citep{scarselli2008graph,bruna2013spectral,defferrard2016convolutional} have achieved significant success in recent years, driving the development of novel architectures and investigations into their theoretical foundations.
% Along such advances, associated limitations have also come to light; most notably %representational collapse, 
% oversmoothing and oversquashing \citep{shao2023unifying, rusch2023survey, jin2025oversmoothing, attali2024rewiringtechniquesmitigateoversquashing}, which effectively ``collapse'' feature vectors into uninformative ones. These phenomena negatively impact the performance of GNNs on long-range tasks, as they are closely tied to information flow during backpropagation \citep{arroyo2025vanishing}.  %While their precise origins remain an active area of research \citep{arnaiz-rodriguez_oversmoothing_2025}

%MPNNs are a class of GNNs in which node representations are updated by aggregating messages from neighboring nodes through learnable functions \citep{gilmer2017neural}. This design offers flexibility but comes with limitations. In particular, their performance on long-range tasks is impacted, as the architecture intrinsically suffers from oversmoothing and oversquashing \citep{li2018deeper, alon2021on}, which have been shown to be consequences of the same contractive effect (vanishing gradients) during training \citep{arroyo2025vanishing}.

Recent work has interpreted GNN layers through the lens of continuous dynamics, viewing propagation as the discretization of an evolution equation on the graph \citep{poli2019graph}. This view has motivated diffusion-based formulations that recast message passing as the discretization of partial differential equations (PDEs) on graphs, enabling principled control over information propagation through learnable diffusivity \citep{chamberlain2021beltrami} or global attention mechanisms \citep{wu2023advective, wu2025transformers}. 
Although the continuous GNN diffusion perspective provides a principled foundation for designing expressive architectures, it often leads to high computational complexity due to  discretization. %{\color{teal}This is especially limiting on long-range tasks: when propagation follows the fixed diffusion geometry of the input graph, information may need to undergo many repeated mixing steps before reaching task-relevant nodes, increasing the risk of contraction, oversmoothing, and bottlenecked communication.}

Another line of work addresses long-range propagation by dynamically rewiring the graph, for example by adding task-dependent edges or modifying the connectivity structure \citep{gutteridge2023drew, barbero2023locality}. While effective, such methods can introduce additional preprocessing cost, alter the original edge set and they may change the semantics of the input graph.

Conversely, spectral GNNs \citep{defferrard2016convolutional, kipf_semi-supervised_2017} offer simple and computationally efficient architectures by leveraging the Fourier domain of the graph, only requiring to learn the fixed spectral filters' weights. Recent work on stabilization of the information flow across layers \citep{gravina_adgn, hariri2026return} has shown that stabilized spectral GNNs can outperform most other GNNs on long-range tasks, and it is believed that they possess competitive generalization qualities \citep{10.5555/3546258.3546530}. Such spectral GNNs, however, still rely on a \emph{fixed} graph Laplacian $L$, limiting their ability to adapt information flow to task-specific structure. In particular, \emph{they lack a mechanism for locally steering information propagation}.

In this work, we introduce a task-adaptive graph Laplacian for spectral GNNs, producing a rewiring-like effect without actually rewiring the graph. Our key idea is to learn a node-wise density $\mu$ that determines where information flows more easily on the graph. Instead of propagating with the fixed Laplacian $L$, we construct a weighted Laplacian $L_\mu$ by reweighting each existing edge according to the learned densities at its endpoints. This is analogous to diffusion in an inhomogeneous medium, where spatial variations in the medium locally guide propagation. 
%Expanding the corresponding weighted diffusion operator reveals a first-order drift-like term induced by variations of $\mu$, so t
The induced dynamics acquire a learned propagation bias without introducing an explicit vector field, thus mitigating traffic through contractive bottlenecks and breaking spectral degeneracies, all while leaving the graph topology unchanged (Figure \ref{fig:SBM-heat-example}). We learn $\mu$ upstream and use the induced operator $L_\mu$ in place of the fixed Laplacian inside ChebNet-style spectral filters. This yields $\mu$-ChebNet and $\mu$-Stable-ChebNet, which combine task-adaptive propagation with efficient polynomial filtering. Rather than learning separate node-to-node communication weights as in attention, our method learns one density value per node. This induces edge weights on the fixed graph and biases transport toward task-relevant regions while keeping the model lightweight.
\begin{figure}[t]
    \centering

    % Left column: two stacked figures
    \begin{minipage}[h!]{0.62\textwidth}
        \vspace{0pt}
        \centering

        \begin{subfigure}[t]{\linewidth}
            \centering
            \includegraphics[width=\linewidth]{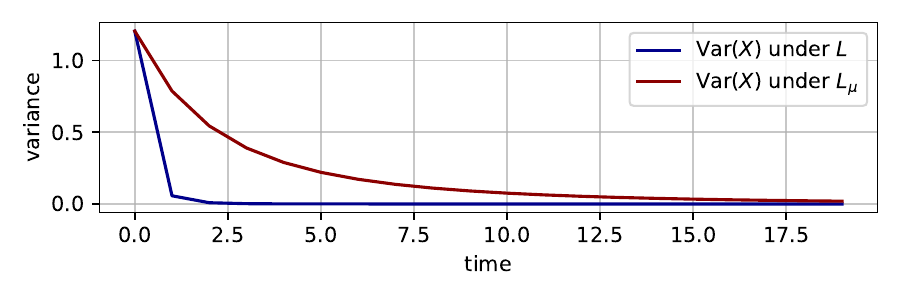}
            \vspace{-20pt}
            \caption{Variance of the signal along heat diffusion driven by $L$ and $L_\mu$.}
            \label{fig:top}
        \end{subfigure}

        \begin{subfigure}[t]{\linewidth}
            \centering
            \includegraphics[width=\linewidth]{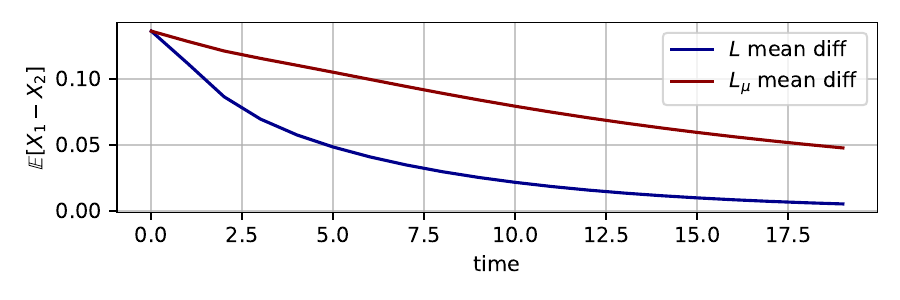}
            \vspace{-20pt}
            \caption{Difference in the mean of the signal between the SBM's two communities along heat diffusion driven by $L$ and $L_\mu$.}
            \label{fig:bottom}
        \end{subfigure}
    \end{minipage}
    \hfill
    % Right column: global caption
    \begin{minipage}[t]{0.34\textwidth}
        \vspace{-100pt}
        \caption{
            Filtering via heat diffusion using $L$ (blue) and $L_\mu$ (red) on the Stochastic Block Model, where $\mu$ has a sharp variation at the community interface. The initial signal is white Gaussian noise. We observe slower variance contraction (a) under $L_\mu$, indicating successful oversmoothing mitigation, as well as slower decay of the mean (b) between the signal on different communities $X_1$ and $X_2$, respectively. This simple reweighting of the graph Laplacian significantly changes the information loss.
        }
        \label{fig:SBM-heat-example}
    \end{minipage}
    \vspace{-10pt}
\end{figure}

Our work makes the following contributions:
\begin{enumerate}
\item \textbf{A task-adaptive Laplacian for spectral GNNs.}
We derive a discrete Laplacian $L_\mu$ whose entries depend on a learnable node-wise density $\mu$. This replaces the fixed graph Laplacian with a learned propagation operator that preserves the original topology.

\item \textbf{Spectral control of propagation dynamics.}
We characterize how the learned density $\mu$ reshapes the spectrum of $L_\mu$. Under a fixed normalization of $\mu$, this corresponds to redistributing spectral energy across modes. This changes the relative timing with which graph Fourier components are smoothed, providing a mechanism for delaying overmixing and modulating bottleneck-sensitive propagation.

    \item \textbf{A new family of spectral GNNs.} 
    % We introduce $\mu$-ChebNet and $\mu$-Stable-ChebNet, two instantiations of a broader idea: replacing the fixed Laplacian in a spectral GNN by the learned operator $L_\mu$. These models bridge spectral and message-passing viewpoints: they retain the efficiency of polynomial spectral filtering, while propagation remains a local diffusion process along existing edges.
%A new family of spectral GNNs.}
We introduce $\mu$-ChebNet and $\mu$-Stable-ChebNet, two architectures that jointly learn (i) propagation dynamics through the density $\mu$ and (ii) spectral filters through Chebyshev polynomials. These models bridge spectral and message-passing viewpoints: they retain the efficiency of polynomial spectral filtering, while propagation remains a local diffusion process along existing edges, now governed by a learned task-adaptive Laplacian rather than the fixed graph Laplacian.
    \item \textbf{Extensive empirical evidence and interpretability.}
Across synthetic long‑range reasoning tasks and real-world benchmarks, our models are competitive with strong spectral baselines, with the most consistent gains obtained when the learned Laplacian is combined with stability constraints. The learned node-wise density $\mu$ also provides an interpretable graph signal, revealing where propagation is strengthened or attenuated for a given task.
\end{enumerate}

\paragraph{Outline}  Section \ref{sec:background} reviews background on graph Laplacians, diffusive dynamics, and Chebyshev spectral filters. Section \ref{sec:BE-graph-Laplacians} introduces the weighted Laplacian on graphs, analyzes its spectral properties, and shows how it enhances standard spectral GNNs, leading to our $\mu$-ChebNet architecture whose results we present in Section \ref{seq:experiments}. The latter benchmarks our model on synthetic and real-world tasks. We conclude by illustrating how the learned weight function provides a simple and interpretable mechanism for task-dependent routing.
%To the best of the authors' knowledge, this is the first instance of a mathematical bridge between MPNNs and spectral GNNs. Indeed, just like in MPNNs, during learning the model constructs a way to locally aggregate information from neighbors in order to update node features. However in our case, this mechanism is done upstream of the spectral GNN.
%Building on this insight, we propose a novel architecture that parametrizes the  with a shallow GCN and jointly learns a downstream ChebNet. Extensive experiments demonstrate that our approach outperforms strong baselines on standard benchmarks, achieving state-of-the-art results. Furthermore, our theoretical analysis explains how the new diffusion mechanism enhances GNN performance. Interestingly, our mechanism bridges the gap between MPNNs with learnt local aggregation mechanism and the simple signal processing architecture of spectral GNNs. Finally, the learned  provides an interpretable, attention-like signal that sheds light on how the downstream GNN solves specific tasks, which we illustrate through case studies.

\section{Related Work }
\paragraph{Mitigating Oversmoothing \& Oversquashing.} %MPNNs are a class of GNNs in which node representations are updated by aggregating messages from neighboring nodes through learnable functions \citep{gilmer2017neural}. This design offers flexibility but comes with limitations. In particular, their performance on long-range tasks is impacted, as the architecture intrinsically suffers from oversmoothing and oversquashing \citep{li2018deeper, alon2021on}, which have been shown to be consequences of the same contractive effect (vanishing gradients) during training \citep{arroyo2025vanishing}.
Recent works propose PDE-inspired propagation dynamics, with an emphasis on controlling information flow through diffusion and transport mechanisms. 
GRAND casts GNN layers as discretizations of diffusion equations on graphs, where the choice of diffusivity and numerical scheme shapes propagation \citep{chamberlain2021grand}. 
In \citet{wusupercharging}, this view is extended through advective diffusion: global attention is interpreted as a non-local diffusion mechanism that captures latent interactions, while local message passing acts as an advection-like term that keeps the dynamics tied to the observed graph topology. 
While expressive and theoretically motivated, such approaches typically rely on attention-based coupling or the solution of continuous-time dynamics, which can introduce substantial computational overhead. 
Graph Transformers (GTs) \citep{rampavsek2022recipe} pursue a related goal from an architectural perspective, combining local aggregation with global attention to improve long-range communication. 
However, standard global attention scales quadratically in the number of nodes, motivating sparse GTs \citep{shirzad2023exphormer, shirzad2024even, chen2022nagphormer} and dynamical graph-rewiring schemes \citep{arnaiz2022diffwire}.

\paragraph{Efficient Spectral GNNs.} Spectral GNNs (such as ChebNets \citep{defferrard2016convolutional}) have undergone a recent revival in relevance for both their computational efficiency and ability to capture long-range interactions via constructive graph Laplacian-based filters when stabilized \citep{gravina_adgn}. These models are simple to implement and analyze due to the rich tools of spectral graph theory \citep{chung1997spectral}, and were also shown to consistently outperform MPNNs, graph rewiring techniques, and GTs on several long-range tasks \citep{hariri2026return}. 
%,ChebNets have been shown to be prone to instabilities as layers stack. \citet{hariri_return_2025} have recently overcome this issue by proposing enforcing volume-preservation of features \citep{gravina_adgn}. This Stable-ChebNet architecture was also shown to consistently outperform MPNNs, graph rewiring techniques, and GTs on several long-range tasks. %, highlighting the continued relevance of spectral models. 
%{\color{teal} The generalization power of spectral GNNs to graph structure unseen at training is also largely still an open question. Experiments, including our own (Section \ref{seq:experiments}), suggest that spectral GNNs transfer to unseen graphs and this is supported by the theoretical findings of \cite{wang2025generalization} and \cite{10.5555/3546258.3546530}.}
%While this line of work improves the stability and expressivity of spectral filters, it usually keeps the underlying graph Laplacian fixed. A complementary direction instead modifies the Laplacian itself to encode additional structure in the propagation operator.

\paragraph{Modified Graph Laplacians.} Rather than designing new architectures, there has been considerable research done in finding different graph Laplacians, with hopes of enriching information propagation across the graph. Modern approaches include degree-dependent penalties \citep{gong2015deformed}, attractive \& repulsive effects modeling \citep{li2023signed}, complex phases for orientation \citep{zhang2021magnet, fiorini2023sigmanet}, or control over spectral scaling \citep{maskey2023fractional}. All the aforementioned approaches remain constrained by the given graph structure and use propagation operators that are \emph{fixed} independently of the task. In contrast, our adaptive Laplacian $L_\mu$ introduces a learnable node-wise     weight function $\mu$, yielding spatially varying, task-dependent transport while preserving the original graph topology. Rather than shortening paths by adding edges, this changes the effective transport geometry on the existing graph, creating preferred diffusion routes that can reduce the amount of repeated mixing required for long-range communication.

%The deformed Laplacian \cite{gong2015deformed} introduces degree-dependent penalties, while signed Laplacians \cite{li2023signed} separate positive and negative interactions to model attractive and repulsive effects. Magnetic Laplacians \cite{zhang2021magnet} further incorporate edge orientation through complex phases, and related variants such as magnetic signed Laplacians and SigMaNet \cite{fiorini2023sigmanet} combine directionality with signed or weighted interactions. Fractional Laplacians \cite{maskey2023fractional}, by contrast, replace $L$ with $L^\alpha$, giving global control over spectral scaling.

\section{Background}\label{sec:background}
In this section we recall elements of calculus on graphs, which will allow us to formulate drift-diffusion dynamics on graphs and derive both the classic notion of a graph Laplacian $L$, as well as the task-adaptive weighted Laplacian $L_\mu$. We will also briefly recall the ChebNet architecture, as well as motivate this specific choice for the architecture we will subsequently introduce.
\subsection{Calculus on Graphs} \label{ssec:graph-calculus}
Let $\mathcal G = (V,E)$ be an undirected graph with $|V| = n$ vertices and $|E| = m$ edges. We denote by $\mathcal{H}_V$ the Hilbert space of square-integrable \emph{node signals }$f:V \rightarrow \mathbb{R}$, with corresponding scalar product $\langle f, g\rangle := \sum_{i=1}^nf(i)g(i)$, for $i$ running over the nodes in $V$. Likewise, $\mathcal{H}_E$ will denote the Hilbert space of square-integrable \emph{edge signals} $G : E \rightarrow \mathbb{R}$. To define signed edge differences, we fix an arbitrary orientation of each undirected edge. This orientation is only a bookkeeping device and does not affect the Laplacian obtained below. The corresponding scalar product is given by $\llangle F, G \rrangle := \sum_{e=1}^m F(e) G(e)$, for $e$ running over the edges in $E$. We will reserve the use of capital letters for such edge signals. 

For an oriented edge $e=(i,j)$, the \emph{gradient operator on $\mathcal G$} is denoted  $\nabla^\mathcal G:\mathcal H_V \to \mathcal H _E $ and acts on a node signal $f\in\mathcal H_V$ as 
\begin{IEEEeqnarray}{rCl}\label{eq:graph-grad}
 \nabla^\mathcal G_{ij} f & := & f(j)-f(i)
\end{IEEEeqnarray}
for two connected nodes $i,j\in V$. The adjoint of $\nabla^\mathcal G$ with respect to $\langle\cdot,\cdot\rangle$ as above is given by ${\nabla^\mathcal G}^\star:\mathcal H_E \to \mathcal H_V$ acting on $F\in\mathcal H_E$ as
\begin{equation}\label{eq:graph-grad-adj}
    ({\nabla^\mathcal G}^*F)(i)=\sum_{e=(j,i)}F(e)-\sum_{e=(i,j)}F(e),
\end{equation}
where the first sum is over edges oriented toward $i$ and the second over edges oriented away from $i$.
\begin{proposition}
\label{prop:grad-div-adjoint}
For any $f\in\mathcal H_V$ and $F\in\mathcal H_E$, we have
$\llangle \nabla^\mathcal G f,F\rrangle=\langle f,{\nabla^\mathcal G}^*F\rangle.$
\end{proposition}
\begin{proposition}
\label{prop:graph-laplacian}
The positive graph Laplacian induced by the gradient is 
$L := \textnormal{div}^\mathcal G\nabla^\mathcal G$, where the divergence operator is given by $\textnormal{div}^{\mathcal G}:=-{\nabla^\mathcal G}^*$. We immediately notice that this is analogous to the Euclidean construction of the Laplace operator $\Delta:=\textnormal{div}\nabla$. Moreover, we have
\begin{equation}
\label{eq:graph-laplacian}
L=D-A,
\end{equation}
where $A$ is the adjacency matrix of $\mathcal G$ and $D$ is the degree matrix with entries $D_{ii}=\sum_{j=1}^n A_{ij}$. This matches the well-known definition of the \emph{combinatorial graph Laplacian} \citep{chung1997spectral}. 
\end{proposition}

Eq. \eqref{eq:graph-laplacian} is particularly useful, as such an $L$ encodes
the connectivity structure of the graph, is symmetric, positive semidefinite, and admits an eigenvalue decomposition of the form
%\begin{equation}\label{eq:basis}
$L=U \Lambda U^\top$,
%\end{equation}
where $\Lambda = \textnormal{diag}(\lambda_0, \lambda_1, \dots, \lambda_n)$ is the diagonal matrix of ordered eigenvalues $0\leq \lambda_0 \leq \lambda_1 \leq \dots \leq \lambda_{n-1}$ and $U$ is an orthogonal matrix of eigenvectors.
\subsection{Weighted Diffusions Steer \& Propagate Information}\label{ssec:weighted-diffusions}
The diffusion equation models the evolution of a quantity $u: \mathbb R^d \times \mathbb R_{\geq0} \to \mathbb R$ that spreads out depending on a diffusion term. Starting from the conservative continuity equation
\begin{equation}
    \partial_t u(x,t) = -\nabla\cdot j(x,t),
\end{equation}
where $j(x,t) = -\mu(x)\nabla u(x,t)$ where $\mu: \mathbb R^d \to \mathbb R_+$ is the spatially-dependent diffusivity, we get the following evolution
\begin{equation}\label{eq:weighted-adv-diff}
    \partial_t u(x,t) = \mu(x)\Delta u(x,t) +\nabla \mu(x) \cdot\nabla u(x,t).
\end{equation}
When $\mu \equiv 1$, Eq.\eqref{eq:weighted-adv-diff} reduces to
$\partial_t u=\Delta u$, with solution $u(t)=e^{t\Delta}u(0)$. When $\mu$ varies in space, however, the right-hand side of \eqref{eq:weighted-adv-diff} develops a first-order term in $u$ and therefore has the form of a drift or advection contribution. In this sense, an apparently purely diffusive equation can induce drift-like dynamics once diffusion takes place in an inhomogeneous medium.

\begin{framed}
This observation is the basis of our approach: instead of prescribing an external drift, we \emph{learn the geometry that induces it}. A positive node-wise weight function $\mu$ biases diffusion such that drift-like transport emerges from the operator itself. Learned from data, $\mu$ becomes a task-dependent geometric bias that steers information flow through the graph.
\end{framed}
\subsection{The ChebNet Architecture}
\label{sec:chebnet}

The previous section suggests a route to task-adaptive propagation: rather than solving an explicit advection-diffusion PDE or introducing attention-like transport mechanisms, one may seek to encode the transport bias directly into the diffusion operator $\Delta \mapsto \mu\Delta$. On graphs, the natural diffusion operator is the Laplacian $L$ \eqref{eq:graph-laplacian}. This makes spectral GNNs a natural setting for our construction, since their propagation rules are built directly from $L$ and its spectrum.
    
ChebNet is particularly well suited to this perspective. It defines graph convolutions through spectral 
filters of the Laplacian, while avoiding the cost of an explicit eigendecomposition through a polynomial 
approximation. 
\begin{comment}
{\color{teal}Consider the eigendecomposition of the Laplacian as in Eq. \eqref{eq:basis}. For node 
signals $x,y \in \mathbb{R}^n$, the graph convolution is defined in the Fourier basis of $L$ as
\begin{equation}
    x *_\mathcal G y 
    =
    U\big((U^\top x) \odot (U^\top y)\big),
\end{equation}
where $\odot$ denotes the element-wise Hadamard product. More generally, a signal 
$x \in \mathbb{R}^{n \times d}$ can be filtered spectrally as
\begin{equation}
    y = U g_\theta(\Lambda) U^\top x,
\end{equation}
where $g_\theta$ is a learnable spectral filter.}
\end{comment}
ChebNet approximates this filter using Chebyshev polynomials, yielding
\begin{equation}
    g_\theta(\Lambda)
    \approx
    \sum_{k=0}^{K} \Theta_k T_k(\widetilde{\Lambda}),
    \qquad
    \widetilde{\Lambda}
    =
    \frac{2\Lambda}{\lambda_{\max}} - I,
\end{equation}
where $T_k$ is the $k$-th Chebyshev polynomial and $\widetilde{\Lambda}$ rescales the spectrum to 
the interval $[-1,1]$. Since polynomials commute with the eigendecomposition, the resulting layer can 
be written directly in terms of the Laplacian:
\begin{equation}
    y=\sum_{k=0}^{K}
    T_k(\widetilde{L})x\Theta_k,
    \qquad
    \widetilde{L}
    =
    \frac{2L}{\lambda_{\max}} - I .
\end{equation}
This form is crucial: the filter depends only on repeated applications of the Laplacian, and can therefore 
be evaluated in $O(Km)$ time without computing the eigenvectors explicitly~\cite{defferrard2016convolutional}.

Consequently, any principled modification of the Laplacian immediately induces a corresponding 
modification of the propagation geometry, while leaving the spectral architecture essentially unchanged. 
This makes ChebNet an ideal backbone for our construction: once the learned weight function $\mu$ is used to 
bias the graph Laplacian $L\mapsto L_\mu$, the resulting geometry can be inserted directly into the Chebyshev 
filtering pipeline.

\section{From Weighted Diffusion to a Node-Weighted Graph Laplacian}\label{sec:BE-graph-Laplacians}
We now show how to construct biased diffusions on graphs, as done in Section \ref{ssec:weighted-diffusions}. 
\subsection{Weighted Graph Laplacian}\label{ssec:weighted-graph-laplacian}
Indeed, introducing an external drift term of the form $v(x)\cdot\nabla u(x,t)$, where $v:\mathbb R^d\to\mathbb R^d$ is a vector field, into Eq. \eqref{eq:weighted-adv-diff} is not well-defined on graph domains, unlike the continuous setting \cite{desbrun2005discrete}. 
%On a graph, however, node signals live on vertices, whereas discrete gradients naturally live on edges. Defining a vector field therefore requires additional non-canonical choices, such as an orientation. 
Instead, we turn to Dirichlet forms, which admit both $\mathbb R^d$ and graph analogs. In finite dimensions, once an inner product $\langle\cdot,\cdot\rangle$ is fixed, a Dirichlet form $\mathcal D$ is in one-to-one correspondence  \citep{fukushima2011dirichlet} with a self-adjoint positive semidefinite Laplacian $L$ through the representation
\begin{equation}
    \mathcal D(f,g) = \langle f,Lg\rangle,
\end{equation}
where $f,g\in L^2(\mathbb R^d; \mathbb R)$ in the continuous case, and $f,g\in \mathcal{H}_V$ in the graph case. 

\paragraph{Continuous Case.} Let $\Omega \subset \mathbb R^d$ be a bounded open domain, and let $f,g:\Omega\subset\mathbb{R}^d\to\mathbb{R}$ be smooth functions vanishing on $\partial\Omega$. For a positive weight function $\mu:\Omega \to \mathbb R_+$, the Dirichlet form given by 
\begin{IEEEeqnarray}{rCl}\label{eq:smooth-dirichlet-form}
    \mathcal D_\mu(f,g) & := & \frac{1}{2}\int_\Omega \mu(x)\nabla f(x)\cdot \nabla g(x) dx  \nonumber \\
    & = & \langle f,\tilde{\Delta}_\mu g\rangle_{L^2( dx)}, 
\end{IEEEeqnarray}
is in one-to-one correspondence to the second-order operator 
\begin{equation}\label{eq:cont-weighted-laplacian}
    \tilde \Delta_\mu f=-\frac12\mu\Delta f-\frac12\nabla\mu\cdot\nabla f,
\end{equation}
when considering the usual $L^2$ inner product. This particular choice of inner product allows us to recover the emerging first-order drift term, previously derived in Eq. \eqref{eq:weighted-adv-diff}. We now turn to the case of graphs.

\paragraph{Graph Case.}
Analogously to the continuous case above, the Dirichlet form on a graph $\mathcal G$ with positive node-wise weight function $\mu:V\to\mathbb{R}_+$ is given by 
\begin{equation}
\label{eq:graph-dirichlet-form}
\mathcal D_\mu^\mathcal G(f,g)=\frac12\sum_i \mu_i\sum_{j\sim i}\nabla^\mathcal G_{ij}f\,\nabla^\mathcal G_{ij}g,
\end{equation}
where $\nabla^\mathcal{G}$ is defined as in Eq. \eqref{eq:graph-grad}. In the following result, we derive the corresponding graph operator $L_\mu$ in closed form, associated to the edge signal inner product $\llangle\cdot,\cdot\rrangle$ from Section \ref{sec:background}.

\begin{theorem}[Closed form of $L_\mu$]
\label{prop:graph_BE_closed_form}
Let $\mu:V\to\mathbb{R}_+$ and let $\mathcal D_\mu^\mathcal G$ be the weighted graph Dirichlet form defined in \eqref{eq:graph-dirichlet-form}. Then there exists a unique symmetric positive semidefinite matrix $L_\mu$ such that $\mathcal D_\mu^\mathcal{G}(f,g)=f^\top L_\mu g$. \textbf{Moreover, $L_\mu$ is the weighted graph Laplacian}
\begin{equation} 
    L_\mu=D_\mu-A_\mu,
\end{equation}
where
$A_\mu= M_\mu\odot A$,  $(M_\mu)_{ij}=\frac{\mu_i+\mu_j}{2}$,
and
$D_\mu=\operatorname{diag}(A_\mu\mathbf 1)$.
\end{theorem} 
\begin{proof}
    See Appendix \ref{app:BE_graph_derivation}.
\end{proof}
Theorem \ref{prop:graph_BE_closed_form} is the weighted analogue of Eq.~\eqref{eq:graph-laplacian}: the standard Laplacian $L=D-A$ is recovered as the special case $A_\mu=A$. The theorem shows that the edge weights are induced by the node-wise weight function $\mu$: each edge $(i,j)$ receives $\frac12(\mu_i+\mu_j)$ from its two endpoint values. In particular, the original graph topology is preserved, since $A_{ij}=0$ implies $(A_\mu)_{ij}=0$. The weight function therefore changes only the ease of transport along existing edges.
The next result makes the analogy with the continuous identity from Eq. \eqref{eq:cont-weighted-laplacian} explicit.
\begin{theorem}[Drift-diffusion decomposition]
\label{prop:graph_BE_decomposition}
Using the graph calculus from Section \ref{ssec:graph-calculus}, the action of $L_\mu$ can be
written as
\begin{equation}\label{eq:graph-weighted-laplacian}
(L_\mu f)_i = \mu_i(Lf)_i - \frac12\sum_{j\sim i} \nabla^G_{ij}\mu\,\nabla^G_{ij}f,
\end{equation}
where $L$ is the graph Laplacian as in Eq. \eqref{eq:graph-laplacian}.
\end{theorem}
\begin{proof}
    See Appendix \ref{app:BE_graph_decomposition}.
\end{proof}
The first term on the right-hand side of Eq. \eqref{eq:graph-weighted-laplacian} is the usual graph diffusion scaled by the local value $\mu_i$. The second term weights each edge difference of the signal by the corresponding edge difference of $\mu$, so it appears only where the learned density and the signal vary. It is therefore a discrete first-order, drift-like correction induced by the variation of $\mu$. In this sense, it biases diffusion toward selected regions of the graph without introducing an explicit vector field and without modifying the adjacency. See Appendices \ref{app:Details-geometry} and \ref{appendix:explain} for details and an additional visualization of the learned density.
\subsection{Spectral properties of the $L_\mu$ Laplacian}
\label{subsec:PropBELap}
We now show that the weight function $\mu$ does more than locally reweight the graph: it reshapes the spectral decay profile of the propagation operator. 
% Throughout the experiments, we normalize $\mu$ to have unit mean,
% $
% \frac1n\sum_i \mu_i=1,
% $
% so that $L_\mu$ reflects a redistribution of diffusion strength across the graph rather than a global rescaling of the operator.

\begin{theorem}[Spectral comparison]
\label{thm:bounds}
Let $0=\lambda_0\le \lambda_1\le \cdots \le \lambda_{n-1}$ be the eigenvalues
of $L$, and let $\lambda_k^{\mu}$ denote the eigenvalues of $L_{\mu}$, ordered
similarly. For each $k\in\{0,\ldots,n-1\}$, there exist positive constants
$c_k^\mu \leq C_k^\mu$, defined in Appendix~\ref{app:spectral_bounds}, such that
\begin{equation}
\lambda_k\,\|\mu\|_1 c_k^\mu
\;\le\;\lambda_k^{\mu}\;\le\;\lambda_k\,\|\mu\|_1 C_k^\mu .
\end{equation}
\begin{comment}
where
\[
m_g =
\min_{\substack{g\in \mathrm{Span}(g_k,\ldots,g_{n-1})\\ \|g\|_2=1}}
\mathbb{E}_{\mu}[p_g],
\qquad
M_f =
\max_{\substack{f\in \mathrm{Span}(f_1,\ldots,f_k)\\ \|f\|_2=1}}
\mathbb{E}_{\mu}[p_f].
\]
\end{comment}
\end{theorem}

\begin{wrapfigure}{r}{0.32\textwidth}
    \vspace{-10pt}
    \centering
    \includegraphics[width=0.30\textwidth]{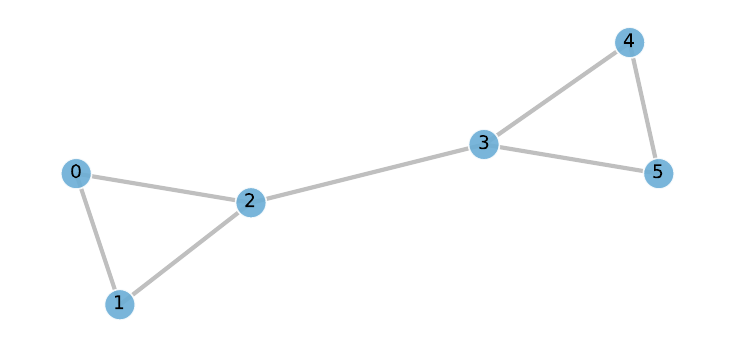}
    \caption{\small Minimal two-community graph.}
    \label{fig:simple-sbm-graph}
    %\vspace{-12pt}
\end{wrapfigure}
The constants $c_k^\mu$ and $C_k^\mu$ are not universal: they quantify how much $\mu$ emphasizes or suppresses the regions where the relevant spectral modes vary. 
Thus, Theorem~\ref{thm:bounds} should be read as a constrained spectral comparison rather than as free control of individual eigenvalues. 
Once the scale of $\mu$ is fixed, changing $\mu$ does not simply rescale diffusion time. 
Instead, it redistributes diffusion strength across the graph, changing the relative decay rates of spectral components, experimentally shown in Figure \ref{fig:SBM-heat-example} on the SBM.
%On bottlenecked graphs, this can preserve cross-bottleneck information long enough for it to propagate before representations inside dense regions collapse.

%We illustrate this effect empirically on barbell graph in Figure \ref{fig:Lmu-eigenvalue-barbell} and on stochastic block models in Figure ~\ref{fig:SBM-heat-example}.
\begin{comment}
\begin{corollary}
\label{cor:upper_bound_1}
Let $\mu$ be such that $\tilde{\mu}(1)=\tilde{\mu}(2)=0$ and
$\tilde{\mu}(k)=\frac{1}{n-2}$ for all $k\notin\{1,2\}$. Then
\[
\lambda_1^{\mu}\le \frac{1}{2(n-2)}\,\|\mu\|_1\,\lambda_1.
\]
\end{corollary}

Corollary~\ref{cor:upper_bound_1} shows that one can strongly reduce the
spectral gap by down-weighting only two nodes, thereby delaying oversmoothing without changing the graph structure.
\end{comment}
% \begin{corollary}
% \label{cor:lower_upper_bound}
%  Let $\tilde{\mu}:=\mu/\|\mu\|_1$ and $\|\mu\|_1=2$ satisfy $\tilde\mu(0)=1/2$, $\tilde\mu(1)=\tilde\mu(2)=0$, and $\tilde\mu(i)=\frac{1}{2(n-1)}+\frac{1}{(n-1)(n-3)}$ for $i\notin \{0,1,2\}$. Then we have
%  \begin{equation}
%  \frac{3}{2}\lambda_{n-1}\le \lambda_{n-1}^{\mu}
%  \end{equation}
%  and 
%  \begin{equation}
%  \lambda_1^{\mu}\le \frac{1}{2}\lambda_1.
%  \end{equation}
% \end{corollary}
% We experimentally validate Corollary \ref{cor:lower_upper_bound} in Section \ref{seq:experiments}.
\paragraph{Avoiding Spectral Redundancy on Bottlenecks}Finally, we illustrate the effect of $\mu$ on a minimal community graph. 
Consider two triangles connected by a single bridge edge, the smallest graph with two dense communities and a bottleneck. The standard Laplacian has spectrum $\{0,\;0.43,\;3,\;3,\;3,\;4.56\}$ where the repeated eigenvalue reflects the graph symmetry: several distinct modes are damped at the same rate under heat diffusion. 
The first non-zero eigenvector separates the two communities: it is nearly constant within each triangle and changes mainly across the bridge. This makes it the spectral component carrying the community-level information.
If we set $\mu=1.5$ on one community and $\mu=0.5$ on the other, the weighted Laplacian has spectrum $\{0,\;0.38,\;1.5,\;2.18,\;4.5,\;5.44\}$. Thus, without changing the edge set, $\mu$ breaks the symmetry of the diffusion operator and separates modes that were previously degenerate. 
The learned density changes the finite-time diffusion regime by making some modes persist longer and others decay faster.
We provide a detailed analysis of this phenomenon in Appendix \ref{app:simple-sbm-deepdive}.

\subsection{A novel class of spectral GNNs}
The explicit form of the weighted graph Laplacian presented in Theorem~\ref{prop:graph_BE_closed_form} matches the form of the combinatorial Laplacian (Eq. \eqref{eq:graph-laplacian}) and can therefore be used in any spectral GNN, with minimal modification to the architecture. 

\paragraph{Learnable Weight Function.} As discussed in Section~\ref{subsec:PropBELap}, the weight function $\mu$ allows control of the spectral properties of $L_\mu$, but it usually remains difficult to design it by hand in practice. We therefore propose a two-step strategy. First, we make a \emph{neural parameterization of $\mu$}. We choose a simple but sufficiently expressive GCN: 
\begin{equation}\label{eq:mu-GCN}
\mu=\textnormal{GCN}(A,X,\Theta),
\end{equation}
where $A$ is the adjacency matrix of $\mathcal G$, $X$ is a feature tensor, and $\Theta$ consists of learnable weights.

\textbf{$\mu$-ChebNet Filter.} Once $\mu$ is learned \eqref{eq:mu-GCN}, the second step is to construct $L_\mu$ using Eq. \eqref{eq:graph-weighted-laplacian}, and use it directly into a downstream spectral ChebNet. This leads to a new family of architectures, with varying parameterization of $\mu$ and downstream ChebNet that can be learned end-to-end to solve a given task. At training time, the weight function $\mu$ will be adapted so the downstream $\mu$-ChebNet solves the task and will thus depend on the task and distribution of training samples. As we shall see below, inspecting $\mu$ as a graph signal reveals much about the strategy employed by the network. 

 %After parameterizing the weight function \mu$ and thus learning the associated Laplacian $L_\mu$, we may incorporate it into an architecture that is particularly well suited to this setting, since its formulation depends only on the Laplacian operator: the ChebNet model \cite{defferrard2016convolutional}.
\textbf{$\mu$-Stable-ChebNet.}
A ChebNet variant that is particularly well suited to $L_\mu$ is the \emph{Stable-ChebNet} architecture \citep{hariri2026return}, which we also implement in the following experiments.
\begin{comment}
The node update rule for Stable-ChebNet is defined as follows:

\[
{X}^{(l+1)} = {X}^{(l)} + \epsilon
\left(
\sum_{k=0}^{K} T_k({L}) {X}^{(l)}
\left(
{W}_k - {W}_k^{\top} - \gamma {I}
\right)
\right)
\]
Here, ${I}$ denotes the identity matrix, $\gamma \in \mathbb{R}$ is a hyperparameter and $\epsilon \in \mathbb{R}^+$ represents the step size of the discretization.  We refer to Theorem~3 of \cite{hariri_return_2025} to show that the $\mu$-Stable ChebNet retains purely imaginary eigenvalues and therefore remains stable. 
\end{comment}

\section{Experiments}\label{seq:experiments}

We evaluate $\mu$-ChebNet and its stable variant on both synthetic and real-world datasets, to probe various aspects of graph reasoning. 
\paragraph{Oversquashing on Barbell Graphs.} The \emph{Barbell dataset} \citep{bamberger2024bundle} synthetically probes oversquashing by requiring nodes in each clique to predict features from the opposite clique, while all cross-clique information must pass through a single bridge edge, 
Thus, even when the message-passing depth is sufficient for reachability, standard MPNNs must compress many source signals through a narrow graph bottleneck, exposing their inability to transmit long-range information without severe distortion.
\paragraph{Results.} We observe in Table \ref{tab:barbell-N5070} that $\mu$-ChebNet significantly outperforms both ChebNet and Stable-ChebNet on barbell graphs of size $N\in\{50,70\}$. Even in the case of $K=9$ hops, where Stable-ChebNet struggles, $\mu$-ChebNet achieves almost zero  mean square error (MSE). For larger graphs of size $N=100$, we see in Table \ref{tab:barbelll-N100} that $\mu$-ChebNet lowers the MSE of Stable-ChebNet by an order of magnitude, while having less layers ($K=15$). Furthermore, we omitted $\mu$-Stable-ChebNet here, as results for $\mu$-ChebNet already demonstrate its improved performance, notably its consistency as graph size increases. 
%{Figure \ref{fig:Lmu-eigenvalue-barbell} \color{teal} supports the spectral re-timing interpretation from Section~\ref{subsec:PropBELap}. By reducing the displayed eigenvalues, $L_\mu$ slows the exponential decay of the corresponding heat-diffusion modes, allowing these components to persist longer before signal collapse.}
\begin{table}[h!]
\centering

% --- Top row: two tables side by side ---
\begin{minipage}[t]{0.47\textwidth}
\centering

\scalebox{0.86}{\begin{tabular}{lccc}
    \toprule
    \textbf{Method} & \textbf{K} & $\mathbf{50}$ & $\mathbf{70}$ \\
    \midrule
    \multirow{2}{*}{ChebNet}
    & $K = 9$  & $0.32 \pm 0.39$ & $1.08 \pm 0.05$ \\
    & $K = 10$ & 0.05 ± 0.00 & $1.08 \pm 0.01$ \\
    \midrule
    \multirow{2}{*}{Stable-ChebNet}
    & $K = 9$  & $0.17 \pm 0.11$ & $0.47 \pm 0.49$ \\
    & $K = 10$ & 0.05 ± 0.00 & 0.06 ± 0.03 \\
    \midrule
    \multirow{2}{*}{\textbf{$\mu$-ChebNet (ours)}}
    & $K = 9$  & \textbf{0.03 ± 0.01} & \textbf{0.02 ± 0.02} \\
    & $K = 10$ & \textbf{0.03 ± 0.01} & \textbf{0.02 ± 0.02}  \\
%    \hspace{0.5em}
%    \multirow{2}{*}{\textbf{\color{teal}+ Stability}}
%    & $K = 9$  & \textbf{?? ± ??} & \textbf{?? ± ??} \\
%    & $K = 10$ & \textbf{?? ± ??} & \textbf{?? ± ??}  \\
    \bottomrule
    
    \end{tabular}}
    \vspace{5pt}
\caption{Mean squared error (MSE) of different ChebNet variants on the oversquashing experiment for barbell graphs of sizes $N = 50,70$.}
\label{tab:barbell-N5070}
\end{minipage}
\hfill
\begin{minipage}[t]{0.47\textwidth}
\centering

\scalebox{0.88}{\begin{tabular}{lcc}
        \toprule
        \textbf{Method} & $\mathbf{K}$ & $\mathbf{100}$ \\
        \midrule
        ChebNet & $K = 20$ & $0.87 \pm 0.05$ \\
        \midrule
        Stable-ChebNet & $K = 20$ & 0.21 ± 0.27 \\
        \midrule
        \textbf{$\mu$-ChebNet (ours)} & $K = 15$ & \textbf{0.02 ± 0.01} \\
%        \textbf{\color{teal}+ Stability}
%        & $K = ??$  & \textbf{?? ± ??}  \\
        \bottomrule
        \end{tabular}}
        \vspace{19pt}
\caption{Mean squared error (MSE) of different ChebNet variants on the oversquashing experiment for barbell graphs of size $N = 100$}
\label{tab:barbelll-N100}
\end{minipage}

%\vspace{1.5em}

% --- Bottom row: horizontal figure ---
% \includegraphics[width=0.95\textwidth]{figures/eigs2.pdf}

% \caption{\color{teal}^$L_\mu$ reduces the displayed eigenvalues relative to the standard Laplacian, slowing the decay of the corresponding heat-diffusion modes and delaying signal collapse.}
% \label{fig:Lmu-eigenvalue-barbell}

\end{table}

\paragraph{Graph Property Prediction.} The \emph{Graph Property Prediction} dataset \citep{corso2020principal} evaluates whether a model can infer global structural quantities of a graph, such as diameter, single-source shortest-path (SSSP) distances, and eccentricity. These tasks test whether a GNN can propagate and integrate information beyond local neighborhoods, exposing limitations in long-range reasoning and structural message diffusion.
 \paragraph{Results.} We observe in Table \ref{tab:graphprop} that $\mu$-ChebNet systematically improves upon the regular ChebNet formulation, and $\mu$-Stable-ChebNet also improves upon Stable-ChebNet, which itself was demonstrated to strongly outperform existing architectures \citep{hariri2026return}. We note that while $\mu$-Stable-ChebNet reports a slightly higher mean than Stable-ChebNet for eccentricity, it remains strongly better than all other models, and lies within Stable-ChebNet's error bars. 

\begin{table}[h!]
\centering
\begin{tabular}{lccc}
\toprule
\textbf{Model} & \textbf{Diameter $\downarrow$} & \textbf{SSSP} $\downarrow$ & \textbf{Eccentricity} $\downarrow$ \\
\midrule
GCN                 & $0.7424 \pm 0.0466$ & $0.9499 \pm 0.0001$ & $0.8468 \pm 0.0028$ \\
GAT                 & $0.8221 \pm 0.0752$ & $0.6951 \pm 0.1499$ & $0.7909 \pm 0.0222$ \\
GraphSAGE           & $0.8645 \pm 0.0401$ & $0.2863 \pm 0.1843$ & $0.7863 \pm 0.0207$ \\
GIN                 & $0.6131 \pm 0.0990$ & $-0.5408 \pm 0.4193$ & $0.9504 \pm 0.0007$ \\
GCNII               & $0.5287 \pm 0.0570$ & $-1.1329 \pm 0.0135$ & $0.7640 \pm 0.0355$ \\
DGC                 & $0.6028 \pm 0.0050$ & $0.1483 \pm 0.0231$  & $0.8261 \pm 0.0032$ \\
GRAND               & $0.6715 \pm 0.0490$ & $-0.0942 \pm 0.3897$ & $0.6602 \pm 0.1393$ \\
A-DGN w/ GCN backbone & $0.2271 \pm 0.0804$ & $-1.8288 \pm 0.0607$ & $0.7177 \pm 0.0345$ \\
\midrule
ChebNet & $-0.1517 \pm 0.0343$ & $-1.8519 \pm 0.0539$ & $-1.2151 \pm 0.0852$ \\

Stable-ChebNet
& $-0.2477 \pm 0.0526$ 
& $-2.2111  \pm 0.0160$ 
& $\mathbf{-2.1043 \pm 0.0766}$ \\

\textbf{$\mu$-ChebNet (ours)} 
& $-0.2075 \pm  0.0634$
& $-2.3149 \pm 0.0301$ 
& $-1.8740 \pm 0.0597$ \\

\quad\textbf{+ Stability}
& $\mathbf{-0.3179 \pm 0.0182}$ 
& $\mathbf{-2.3217 \pm 0.0330}$ 
& $\mathbf{-2.0338 \pm 0.0571}$ \\

\bottomrule
\end{tabular}
\vspace{5pt}
\caption{Mean test set $\log_{10}(\mathrm{MSE})$ and standard deviation averaged over 4 random weight initializations for each configuration on the Graph Property Prediction dataset. The lower the better.}
\label{tab:graphprop}
\end{table}
%We also test $\mu$-ChebNet on the \emph{Graph Property Prediction} benchmark \citep{corso2020principal}, which evaluates both graph-level and node-level classification. The goal is to predict structural quantities such as diameter, single-source shortest-path (SSSP), and eccentricity from graph topology and features by aggregating information over large receptive fields. 

\paragraph{City-Networks Benchmark.} We now evaluate our model on the City-Networks dataset \citep{liang2025towards}, a real-world large-scale node classification benchmark derived from city road networks, with graphs built from OpenStreetMap road junctions and segments and labels based on eccentricity/accessibility. The task consists of node classification of city road junctions, requiring information from distant nodes in large-diameter graphs, thereby directly testing real-world long-range reasoning. %The City-Networks results in Table \ref{tab:city_networks} are not meant to reproduce the original paper’s experimental setup or reported numbers. Instead, we use the dataset to fairly compare ChebNet, Stable-ChebNet, and $\mu$-ChebNet under the same high-receptive-field setting. In particular, we use larger Chebyshev orders K to study long-range propagation behavior that was not the focus of the original paper. The goal is therefore to isolate the effect of using a learned task-adaptive Laplacian.

\paragraph{Results.} We report prediction accuracy (\%) results for four different cities (London, Paris, Shanghai, Los Angeles) in Table \ref{tab:city_networks}, and observe that $\mu$-ChebNet and its stable variant either outperforms or is within error range of Stable-ChebNet. 

\begin{table}[t]
\centering
\begin{tabular}{lcccc}
\toprule
\textbf{Method} &\textbf{London} & \textbf{Paris} & \textbf{Shanghai} & \textbf{Los Angeles} \\
\midrule
ChebNet 
& 39.10 $\pm$ 1.79 
& 34.34 $\pm$ 0.31 
& 40.00 $\pm$ 0.39 
& 41.31 $\pm$ 0.48 \\

Stable-ChebNet 
& 40.11 $\pm$ 0.20
& \textbf{38.11 $\pm$ 0.25} 
& 41.45 $\pm$ 0.27 
& 42.17 $\pm$ 0.16 \\

\textbf{$\mu$-ChebNet (ours) }
& \textbf{41.88 $\pm$ 0.11}
& 35.14 $\pm$ 0.12 
& 41.98 $\pm$ 0.96 
& \textbf{43.33 $\pm$ 0.33} \\

\quad \textbf{+ Stability }
& 41.44 $\pm$ 0.07 
& 37.49 $\pm$ 0.26 
& \textbf{42.63 $\pm$ 0.70} 
& 43.10 $\pm$ 0.60 \\
\bottomrule
\end{tabular}
\vspace{5pt}
\caption{Accuracy (\%) on City-Networks dataset. We use this dataset to fairly compare ChebNet architectures by setting them to the same receptive field $K=10$ and isolating the effect of the weighted Laplacian $L_\mu$. %(mean $\pm$ std over seeds 0--2).
}
\label{tab:city_networks}
\vspace{-20pt}
\end{table}

\paragraph{Open-Graph Benchmark: Proteins.} The \texttt{ogbn-proteins} dataset \citep{ogb} is a biologically meaningful benchmark, with multi-label protein function prediction tasks on a protein-protein interaction network.

\paragraph{Results.}
 We observe in Table~\ref{tab:ogbn-proteins} that $\mu$-ChebNet reaches a score of $79.36\%$, outperforming the MPNN baselines \citep{kipf2016semi, velivckovic2018graph, hamilton2017inductive, xu2019powerful}, and remaining competitive with both Stable-ChebNet \citep{hariri2026return} and Graph Transformer architectures \citep{bo2023specformer, wu2022nodeformer, wu2023sgformer}. 
Although SPEXFORMER obtains the best score among the reported models, our method achieves comparable performance while relying only on ordinary sparse graph operations. Computationally, $\mu$-ChebNet incurs the cost of a shallow GCN used to learn $\mu$, an $O(m)$ construction of $L_\mu$, and a downstream Chebyshev filter whose propagation cost scales as $O(Km)$, up to feature-channel mixing. Hence, on sparse graphs with fixed $K$ and bounded width, the overall complexity is linear in the number of nodes. By contrast, SPEXFORMER reduces the final attention layer to $O(nd^2 + ndc)$ by selecting $c$ attention edges per node, but requires an additional attention-estimation stage. Thus, while both approaches can scale linearly on sparse graphs, our approach obtains task-adaptive propagation through a learned sparse Laplacian and avoids a two-stage attention-sparsification pipeline.
\begin{wraptable}{r}{5.5cm}
\vspace{5pt}
\begin{tabular}{l c}
    \toprule
    \textbf{Model} & \textbf{{ROC-AUC} $\uparrow$} \\
    \midrule
    MLP            & $72.04 \pm 0.48$ \\
    GCN            & $72.51 \pm 0.35$ \\
    ChebNet        & $77.55 \pm 0.43$ \\
    Stable-ChebNet & $79.55 \pm 0.34$ \\
    SGC            & $70.31 \pm 0.23$ \\
    GCN-NSAMPLER   & $73.51 \pm 1.31$ \\
    GAT-NSAMPLER   & $74.63 \pm 1.24$ \\
    SIGN           & $71.24 \pm 0.46$ \\
    NodeFormer     & $77.45 \pm 1.15$ \\
    SGFormer       & $79.53 \pm 0.38$ \\
    \textbf{SPEXPFORMER} & $\mathbf{80.65 \pm 0.07}$ \\
    \midrule
    $\mu$-ChebNet (ours)  & $79.36 \pm 0.41$ \\
   % \hspace{0.5em}{\color{teal}+Stability} & $?? \pm ??$ \\
    \bottomrule
    \end{tabular}
    \caption{Results on \texttt{ogbn-proteins}.}
    \label{tab:ogbn-proteins}
\vspace{-35pt}
\end{wraptable}

\paragraph{Discussion and Future Work.}
% Rather than treating the graph Laplacian as fixed, one can learn the propagation geometry itself. By replacing $L$ with $L_\mu$, $\mu$-ChebNet keeps the efficiency of polynomial spectral filtering while locally modulating information flow on the original graph, yielding preferred propagation routes without attention, rewiring, or explicit edge addition. This suggests that spectral GNNs can recover a form of task-dependent local propagation while retaining the efficiency of polynomial filtering.
\textcolor{black}{Across the experiments, $\mu$-ChebNet is most effective on long-range and bottlenecked tasks, where learning $L_\mu$ improves over fixed-Laplacian ChebNet \& Stable-ChebNet baselines. On real-world benchmarks, it rivals larger architectures such as graph transformers, while remaining a lightweight an end-to-end model. The stability constraint \citep{gravina_adgn, gravina_swan, hariri2026return} yields mixed gains, improving some benchmarks while leaving others unchanged or slightly worse. This suggests a nontrivial interaction with the learned density $\mu$, whose theoretical and empirical role remains open.}

We believe that further important future research lies in deriving tighter eigenvalue bounds for $L_\mu$, evaluating $L_\mu$ with other downstream spectral architectures besides ChebNet-style filters, and ablating over alternative parameterizations of $\mu$ beyond the GCN considered in this work. 

\section{Conclusion}
In this work, we achieve guided diffusion dynamics on graphs by learning a data-driven weighting $\mu$ of the graph nodes, resulting in a task-adaptive graph Laplacian $L_\mu=D_\mu-A_\mu$ (Theorem \ref{prop:graph_BE_closed_form}). This produces an end-to-end rewiring-like effect, without changing the graph topology. 
% We introduced $\mu$-ChebNet, a spectral GNN that learns a weighted Laplacian $L_\mu$ instead of relying on the fixed graph Laplacian. The learned density $\mu$ induces a biased propagation geometry: it changes how strongly diffusion uses existing edges, producing a rewiring-like effect without altering the graph topology.
We prove that this learnable weight function $\mu$ allows task-dependent control over the graph's spectrum by changing the relative decay rates of graph Fourier modes (Theorem \ref{thm:bounds}). Moreover, we introduce $\mu$-ChebNet, an upgraded ChebNet architecture incorporating the modulable graph Laplacian $L_\mu$. 
We experimentally observe that $\mu$-ChebNet drastically outperforms current ChebNet architectures on signal propagation through narrow bottlenecks (Barbell Graph). On real-world datasets, $\mu$-ChebNet also performs comparably to state-of-the-art architectures. Our work opens doors to further exploring the role of task-adaptive geometry in machine learning, offering better interpretability at a lower cost.

\paragraph{Impact Statement.} This work advances the theoretical foundations of graph machine learning by introducing a principled operator that proposes data-driven inhomogeneous diffusion on graphs. While our approach is primarily methodological, future work should consider fairness and bias implications when learning weight functions on graphs.
% \paragraph{Code Availability.} The full source code will be made available upon acceptance of this manuscript.
\newpage

\newpage
\bibliography{muChebNet}
\bibliographystyle{plainnat}
\newpage

%%%%%%%%%%%%%%%%%%%%%%%%%%%%%%%%%%%%%%%%%%%%%%%%%%%%%%%%%%%%

\appendix
\section{Illustrating enhanced explainability}
\label{appendix:explain}
We provide further evidence that the learned weight function $\mu$ helps revealing the strategy used by the network to solve tasks. We consider a dataset of ring graphs and a task that consists in predicting the correct class (among $10$) of a query node. This class is encoded as a Dirac in the features of a distant answer node. The query and answer nodes are symmetrically positioned so that there are two long paths of equal lengths connecting them. The features along the first path are all-zero, but the features along the second are i.i.d Gaussian noise. Clearly, agglomerating information along that second path results in accumulating noise and therefore poor decision making.  As can be seen on Figure~\ref{fig:RingTest}, $\mu$ assigns higher values to vertices on the most efficient path from the answer node to the query node — effectively steering signal propagation along meaningful routes, without altering the original graph topology. Specifically, $\mu$ successfully suppresses one of the two paths, concentrating signal flow along a single route. Beyond improving predictive performance, the learned weight functions offer transparent, interpretable insights into how the model routes information across complex structures.

\begin{figure}[h!]
    \centering
       \includegraphics[width= 15 cm, height=5.1cm]{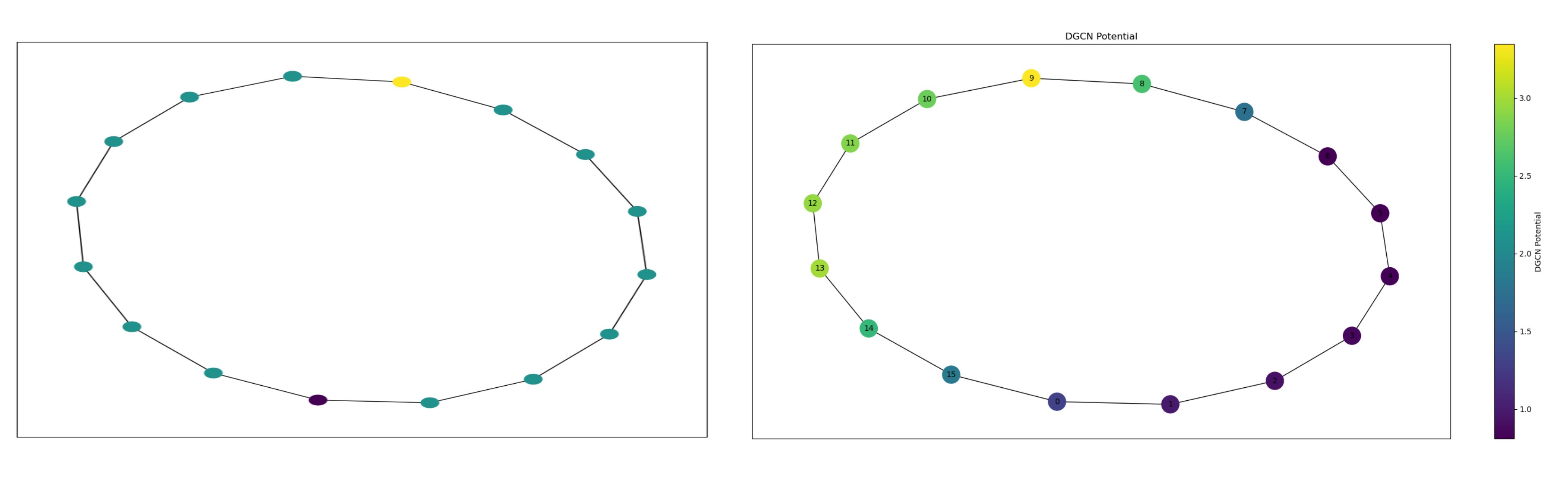}
    \caption{\label{fig:RingTest}Ring graph configuration. (Left) The query node is in purple and the answer node in yellow. Nodes on the rightward path contain noisy features. (Right) The resulting learned weight function has broken the symmetry of the configuration, routing information on the most informative path.}
\end{figure}
\section{Analysis Community Graph}\label{app:simple-sbm-deepdive}
\paragraph{Avoiding Repeated Eigenvalues}
We consider the graph formed by two triangles connected by a single bridge edge. 
This graph is a minimal model of a bottlenecked community structure: each triangle is internally dense, while all inter-community communication must pass through one edge.
For the standard combinatorial Laplacian, the spectrum is
\[
\{0,\;0.438,\;3,\;3,\;3,\;4.562\}.
\]
The repeated eigenvalue at $3$ is caused by the symmetry of the two communities. 
Under heat diffusion, all modes associated with this repeated eigenvalue decay at the same rate. 
The Fiedler eigenvector, associated with the first non-zero eigenvalue, is aligned with the partition between the two triangles and has a large edge difference on the bridge.
We then choose a piecewise constant density $\mu$, with $\mu=1.5$ on one community and $\mu=0.5$ on the other, to yield $\sum_i\mu_i=6=\sum_{i=1}^61$ as for the unweighted case.  
The spectrum of $L_\mu$ becomes
\[
\{0,\;0.38,\;1.5,\;2.18,\;4.5,\;5.44\}.
\]
The weighted operator removes the eigenvalue degeneracy and modifies the eigenvectors. 
Modes whose variations lie mostly in the high-$\mu$ region become more costly and are shifted toward larger eigenvalues, while modes varying mostly in the low-$\mu$ region remain comparatively slower.
Normalizing by the Fiedler eigenvalue gives the relative decay rates
\[
L:
\{0,\;1,\;6.84,\;6.84,\;6.84,\;10.40\},
\]
whereas
\[
L_\mu:
\{0,\;1,\;3.17,\;5.70,\;15.87,\;18.69\}.
\]
This shows that $L_\mu$ does not simply accelerate or slow all modes uniformly. 
Instead, it changes the relative timing of spectral decay: some components become slower relative to the partition mode, while others are damped much faster, as shown in Figure \ref{fig:sbm-deepdive-fourpanels}.
This example highlights the transient nature of the mechanism. 
Both $L$ and $L_\mu$ eventually drive heat diffusion toward the constant eigenspace, so neither operator prevents collapse at infinite time. 
The difference lies in the finite-time regime relevant for GNNs: $L_\mu$ changes which modes remain available before collapse occurs. 
In this sense, the learned density steers propagation by re-timing spectral decay rather than by inducing faster-than-diffusive transport.

\begin{figure*}[t]
    \centering

    \begin{subfigure}[t]{0.48\textwidth}
        \centering
        \includegraphics[width=\linewidth]{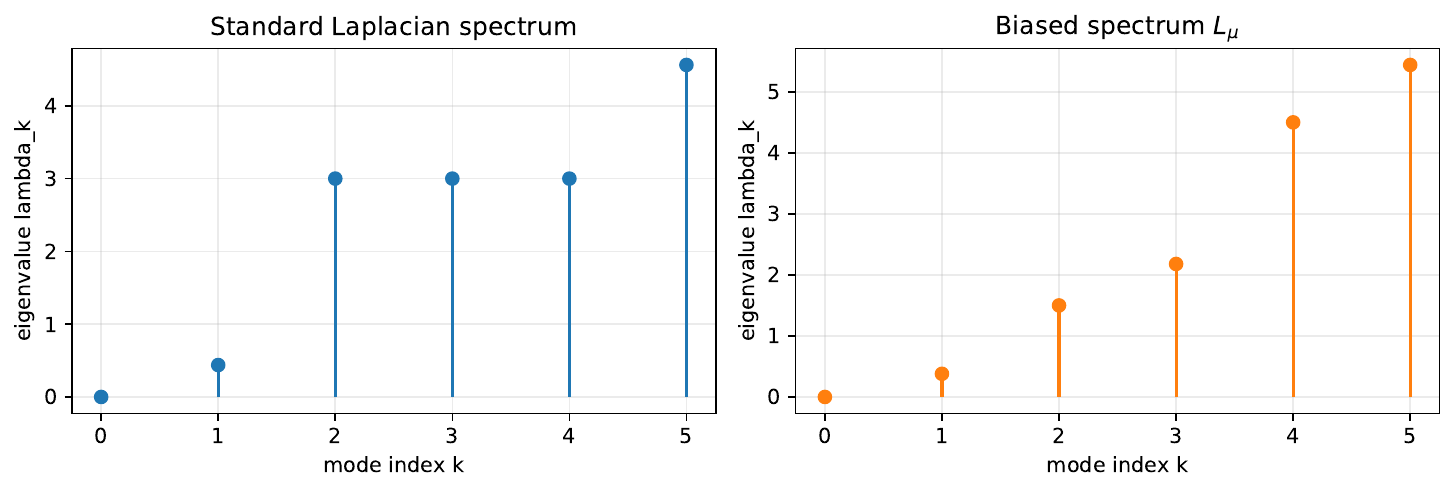}
        \caption{Eigenvalue spectra of the standard Laplacian \(L\) and biased Laplacian \(L_\mu\). The repeated eigenvalue in \(L\) is split under \(L_\mu\), showing symmetry breaking.}
        \label{fig:sbm-graph}
    \end{subfigure}
    \hfill
    \begin{subfigure}[t]{0.48\textwidth}
        \centering
        \includegraphics[width=\linewidth]{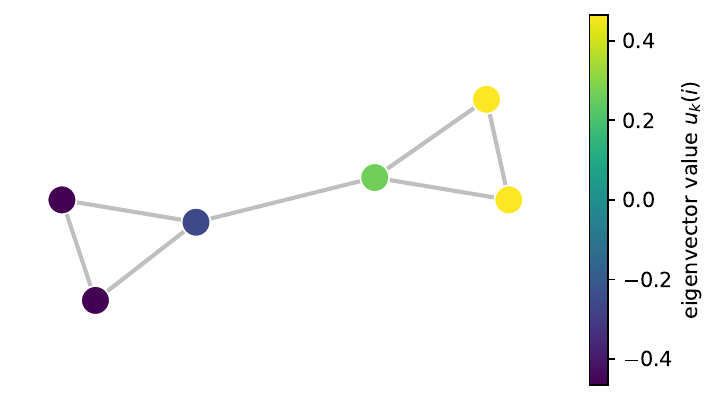}
        \caption{Fiedler mode (\(\lambda_1\)): aligned with the community partition.}
        \label{fig:sbm-fiedler}
    \end{subfigure}

    \vspace{0.6em}

    \begin{subfigure}[t]{0.48\textwidth}
        \centering
        \includegraphics[width=\linewidth]{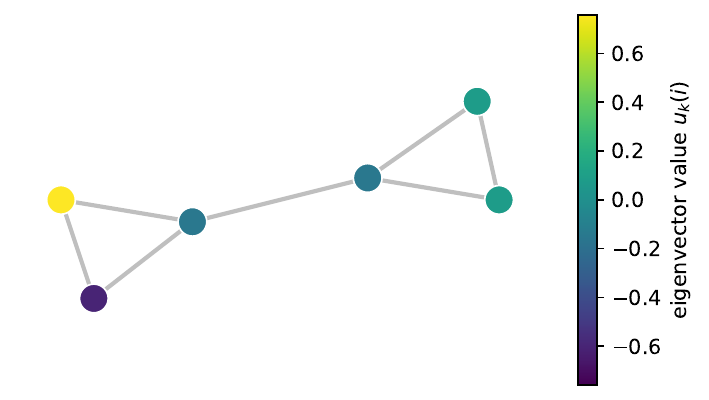}
        \caption{Representative middle-band mode (degenerate band of \(L\)).}
        \label{fig:sbm-mid}
    \end{subfigure}
    \hfill
    \begin{subfigure}[t]{0.48\textwidth}
        \centering
        \includegraphics[width=\linewidth]{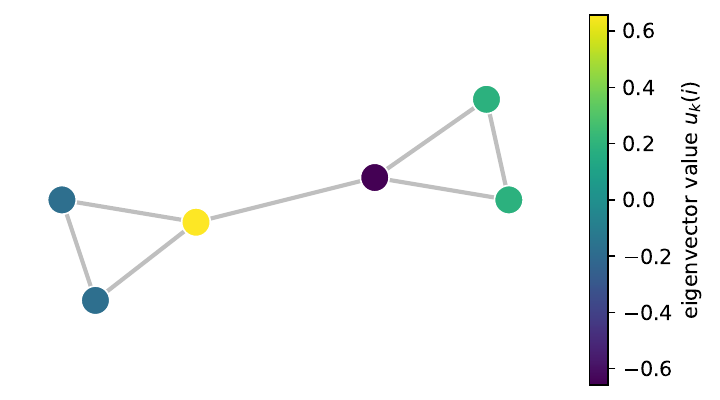}
        \caption{High-frequency mode.}
        \label{fig:sbm-high}
    \end{subfigure}

    \caption{Spectral analysis on the simple community graph.}
    \label{fig:sbm-deepdive-fourpanels}
\end{figure*}
\paragraph{Oversquashing Mitigation via Learned Diffusion Geometry}
\begin{figure*}
    \centering\includegraphics[width=0.5\linewidth]{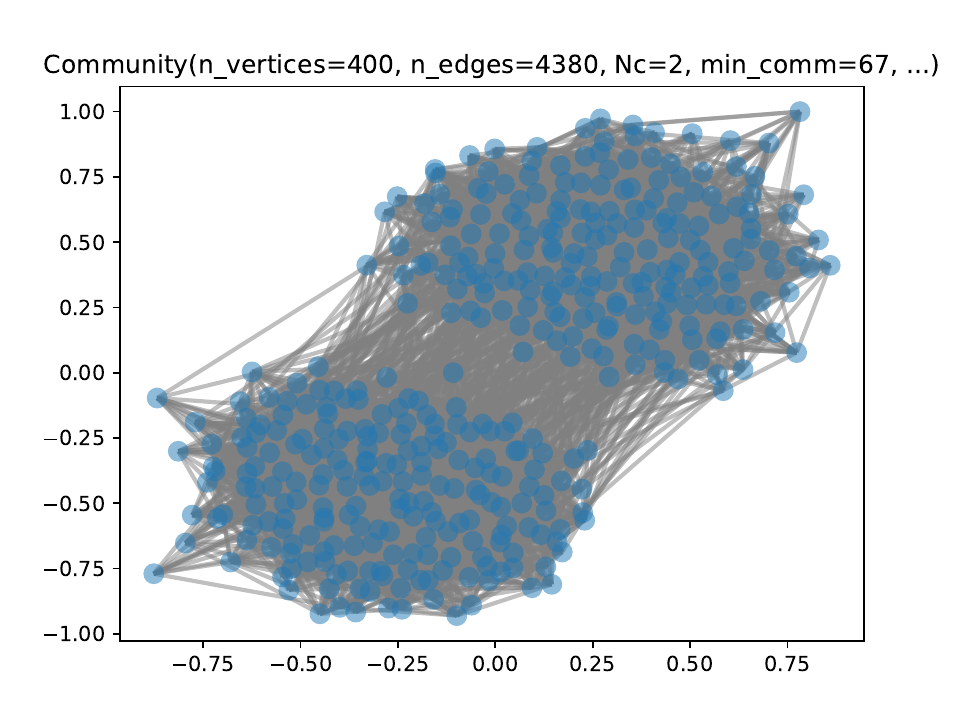}
\caption{2-community Graph}
\label{fig:large-graph}
\end{figure*}
To show how the weighted Laplacian $L_\mu$ helps mitigate oversquashing, we perform the following experiment:
We generate a 2-community graph $C_1$ and $C_2$ with 400 nodes and a strong bottleneck as shown in Figure \ref{fig:large-graph}.
\begin{figure*}
    \centering\includegraphics[width=\linewidth]{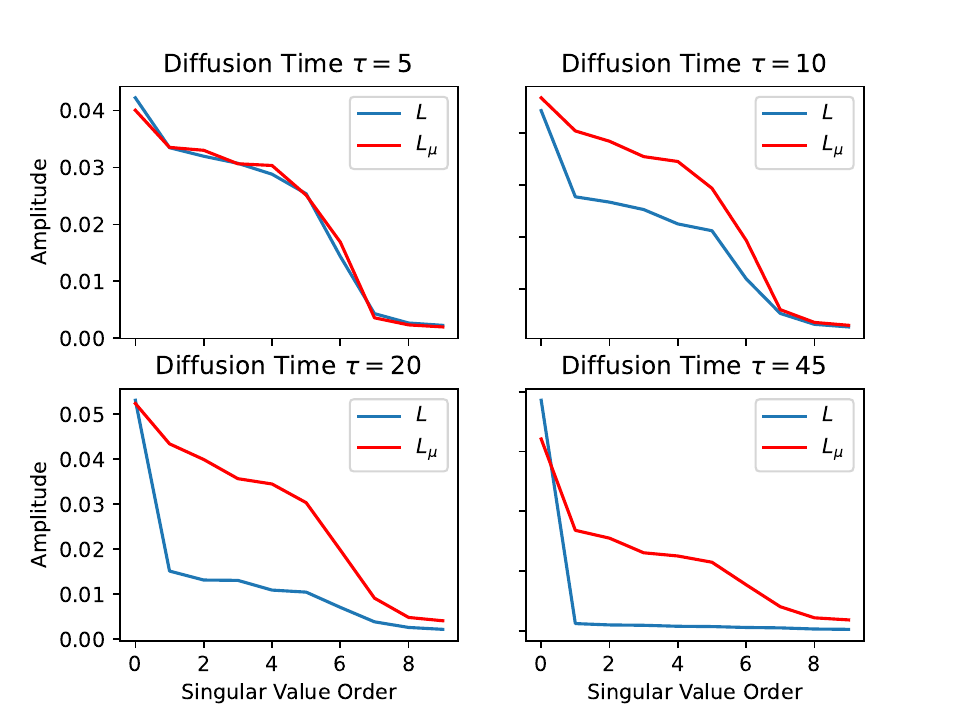}
\caption{Singular values of the propagated feature matrix restricted to the target cluster. While standard diffusion rapidly collapses all but one direction, the learned weighted diffusion geometry opens a transient regime in which multiple independent directions survive long enough to traverse the bottleneck.}
\label{fig:sing-val}
\end{figure*}
We choose $\mu=c_1$ on community $C_1$ and $\mu=c_2$ on community $C_1$, where $c_1, c_2$ are two positive constants. 
Further, we initiate a feature matrix $X_0$ consisting of $k$ orthogonal vectors localized on one community. We then propagate using the heat kernels $e^{\tau L}$ and $e^{\tau L_\mu}$ for the weighted Laplacian applied to $X_0$. As propagation time $\tau$ grows, the signals will transit through the bottleneck to reach the second community. Since the bottleneck is strong, it will limit the flow of information leading to oversquashing. To measure this phenomenon, we track the singular values of the signal restricted to the second community. Under standard diffusion, the effective rank collapses rapidly, indicating severe information compression. In contrast, the weighted Laplacian opens a clear transient regime in which several singular values remain significant, allowing multiple independent directions to traverse the bottleneck before collapse. This demonstrates that the biased diffusion geometry alleviates oversquashing by reallocating spectral mass away from local oscillations and toward bottleneck‑aligned transport, as illustrated in Figure \ref{fig:sing-val}.

\section{Proofs}
\subsection{Proof of Proposition \ref{prop:grad-div-adjoint}}
Taking the definitions of $\nabla^G$ and ${\nabla^G}^\star$ from Eqs. \eqref{eq:graph-grad} and \eqref{eq:graph-grad-adj}, we obtain 
\begin{equation}
\llangle \nabla^\mathcal G f,F\rrangle
=\sum_{e=(i,j)\in E}(f(j)-f(i))F(e)
=\sum_{i\in V}f(i)\left(\sum_{e=(j,i)}F(e)-\sum_{e=(i,j)}F(e)\right)
=\langle f,{\nabla^\mathcal G}^\star F\rangle.
\end{equation}
proving the claim.
\begin{flushright}
    \qed
\end{flushright}
\subsection{Proof of Proposition \ref{prop:graph-laplacian}}\label{app:proof-laplacian-d-a}
For any node $i$,
\[({\nabla^\mathcal G}^*\nabla^\mathcal G f)(i)=\sum_{e=(j,i)}(\nabla^\mathcal G f)(e)-\sum_{e=(i,j)}(\nabla^\mathcal G f)(e).\]

If $e=(j,i)$, then $(\nabla^\mathcal G f)(e)=f(i)-f(j)$. If $e=(i,j)$, then $(\nabla^\mathcal G f)(e)=f(j)-f(i)$, so the second term contributes $f(i)-f(j)$. Hence
\[({\nabla^\mathcal G}^*\nabla^\mathcal G f)(i)=\sum_{j\sim i}(f(i)-f(j))=(D-A)f(i).\]
Therefore ${\nabla^\mathcal G}^*\nabla^\mathcal G=D-A$. The divergence identity follows from $\operatorname{div}^{\mathcal G}:=-{\nabla^\mathcal G}^*$.
\begin{flushright}
    \qed
\end{flushright}

\subsection{Proof of Theorem \ref{prop:graph_BE_closed_form}}
\label{app:BE_graph_derivation}
We first consider a non-negative node-wise weight $\mu: V \rightarrow \mathbb{R}^+$ and the corresponding Dirichlet bilinear form: $\mathcal{D}^G_\mu (f,g) =  \frac{1}{2}\sum_i \mu_i \sum_{j\sim i}\nabla^G_{ij} f \nabla^G_{ij} g.$

% The symmetry with~\eqref{eq:Dirichlet_mu} reveals that the Dirichlet form is a weighted scalar product of the components of the graph gradients of $f$ and $g$ at each node. Again, there exists a unique symmetric positive-definite matrix $L_\mu$ such that $\mathcal{D}_\mu^G (f,g) = f^T L_\mu g= (L_\mu f)^T g$. 
Developing $\mathcal{D}_\mu^G$ we get:

\begin{align*}
\mathcal{D}^G_\mu (f,g) &=  \frac{1}{2}\sum_i \mu_i \sum_{j\sim i}\nabla^G_{ij} f \nabla^G_{ij} g\\
& = \frac{1}{2}\sum_i \mu_i \sum_{j\sim i}
   (g(i)-g(j))(f(i)-f(j)) \\
   & = \frac{1}{2}\sum_i \sum_{j\sim i}\frac{\mu_i S_{ij}+\mu_j S_{ji}}{2} \\
   &= \frac{1}{2}\sum_i \sum_{j\sim i}\frac{S_{ij}(\mu_i +\mu_j)}{2} \\
   & = \sum_{i\sim j}\frac{\mu_i+\mu_j}{2}
        \big( f(i)-f(j) \big)\big( g(i)-g(j) \big)\\%sum over all the edges, so no edges counted twice, so I get rid of factor 1/2, but
   & = \frac{1}{2}\sum_{i} \sum_{j} A_{\mu,ij}
        \big( f(i)-f(j) \big)\big( g(i)-g(j) \big) \\ % because A selects only neighbors, but we need 1/2 to account for (i,j) and (j,i)
        & = \frac{1}{2}\sum_{i} \sum_{j} g(i)A_{\mu,ij}
        \big( f(i)-f(j) \big)-g(j)A_{\mu,ij}
        \big( f(i)-f(j) \big)  \\
        & = \frac{1}{2}\sum_{i} \sum_{j} g(i)A_{\mu,ij}
        \big( f(i)-f(j) \big)-g(i)A_{\mu,ji}
        \big( f(j)-f(i) \big)   \\
        & = \sum_{i} \sum_{j} g(i)\, A_{\mu,ij}\,(f(i)-f(j)) \\
        & = \sum_{i} g(i) \left( D_{\mu,ii} f(i) - \sum_{j\sim i} A_{\mu,ij} f(j) \right) 
\end{align*}

where we have defined $A_\mu = M_\mu \odot A$, with $\odot$ being the Hadamard product and $M_{\mu,ij} = \frac{1}{2}(\mu_i + \mu_j)$. On line 3, we have defined $S_{ij}=\big( f(i)-f(j) \big)\big( g(i)-g(j) \big)$. On line 4, we use the symmetry property of the matrix $S$. A similar reasoning has been used on line 8, since $A_{\mu, {ij}}$ is also symmetric. The weighted adjacency matrix $A_\mu$ preserves the original adjacency and the corresponding degree matrix is defined as usual $D_\mu = \textrm{diag} \bigl( A_\mu \mathbf{1} \bigr)$. Finally, we immediately see that:
\begin{equation}
    D^G_\mu (f,g) = (L_\mu f)^T g \Rightarrow L_\mu = D_\mu - A_\mu
\end{equation}
Which allows us to just simply compute $A_\mu$ and $D_\mu$. 
\begin{flushright}
    \qed
\end{flushright}\subsection{Proof of Theorem \ref{prop:graph_BE_decomposition}}
\label{app:BE_graph_decomposition}
We derive the decomposition of $L_\mu$:
\begin{align*}
(L_\mu f)_i&=\sum_{j\sim i}\frac{\mu_i+\mu_j}{2}(f(i)-f(j))\\
&=\sum_{j\sim i}\left(\mu_i+\frac{\mu_j-\mu_i}{2}\right)(f(i)-f(j))\\
&=\mu_i\sum_{j\sim i}(f(i)-f(j))+\frac12\sum_{j\sim i}(\mu_j-\mu_i)(f(i)-f(j))\\
&=\mu_i(Lf)(i)-\frac12\sum_{j\sim i}\nabla^G_{ij}\mu\,\nabla^G_{ij}f.   
% (L_\mu f)_i&= \sum_{j \sim i}\big(\mu_i-\frac{\mu_i-\mu_j}{2}\big)(f(j)-f(i)) \\ \nonumber
    % &= \mu_i\sum_{j \sim i}\nabla^G_{ij}f+\frac{1}{2}\sum_{j \sim i} \nabla^G_{ij}\mu \nabla^G_{ij}f\\ \nonumber
    % &=\mu_i L f(i) +\frac{1}{2}\sum_{j \sim i} \nabla^G_{ij}\mu \nabla^G_{ij}f
\end{align*}
\begin{flushright}
    \qed
\end{flushright}
\subsection{Proof of Theorem \ref{thm:bounds}}
\label{app:spectral_bounds}
Since $L_\mu$ preserves the graph topology, its effect comes from changing the contribution of local differences $f(u)-f(v)$ (where $u$ and $v$ are two neighboring nodes) to the Dirichlet energy. This section makes this precise through Rayleigh quotients, derives bounds comparing the spectra of $L$ and $L_\mu$, and illustrates how $\mu$ can modulate the spectral gap, the spectral radius, and repeated eigenvalues.
First, we compare $L$ and $L_\mu$ through their Rayleigh quotients, $\mathcal{R}(f)=\frac{f^{\top}Lf}{f^{\top}f}$,
and
$\mathcal{R}_{\mu}(f)=\frac{f^{\top}L_{\mu}f}{f^{\top}f}.$
For $f\in\mathbb{R}^n$, define the normalized local variation profile

$N(f)=\left(\frac{\|\nabla^G f(u)\|_2^2}{f^{\top}f}\right)_{u\in V}$, where
$\|\nabla^G f(u)\|_2^2=\sum_{v\sim u}\bigl(f(u)-f(v)\bigr)^2$,
and the associated probability distribution
$p_f=\frac{1}{\mathbf{1}^{\top}N(f)}\,N(f)$.

In order to complete the proof, let us first state a helpful lemma:
\begin{lemma}[Rayleigh quotient factorization]
\label{lem:Rmu}
Let $f\in\mathbb{R}^n$. Then
\[
\mathcal{R}_{\mu}(f)
=
\|\mu\|_1\,\mathbb{E}_{\mu}[p_f]\,\mathcal{R}(f),
\]
where $\mathbb{E}_{\mu}[p_f]$ is the expectation of $p_f$ under $\mu$
normalized to be a probability distribution.
\end{lemma}
Lemma~\ref{lem:Rmu}
shows that $\mu$ controls the spectrum by changing the cost of local variations:
if $\mu$ is large where $f$ varies strongly, then $\mathcal R_\mu(f)$ increases;
if $\mu$ is small in those regions, it decreases. The spectrum of $L_\mu$ is
therefore controlled by the global scale $\|\mu\|_1$, the spectrum of the
underlying graph through $\mathcal R(f)$, and the alignment between $\mu$ and
the local variation profile $p_f$.

Theorem~\ref{thm:bounds}
shows that shaping $\mu$ gives targeted control over the spectrum, beyond a
global rescaling by $\|\mu\|_1$. In particular, eigenvalues can be increased or
decreased depending on how the potential overlaps with the local variation
profiles of the corresponding eigenvectors.
In the next section, we therefore parameterize $\mu$ and use the resulting Laplacian as an adaptive propagation operator in spectral GNN architectures.

We only prove the upper bound as the same reasoning applies to the lower bound.
Let's define the set of Euclidian subspaces of $\mathbb{R}^n$,
\[
\mathcal{F}_k=\left\{F,\ \dim(F)=k+1 \ \text{and}\ 
\lambda_k=\max\left\{\mathcal{R}(f)\mid f\in F,\ \|f\|_2=1\right\}\right\}
\]
and fix $F\in\mathcal{F}_k$. By Lemma \ref{lem:Rmu}
\[
\max_{\substack{f\in F\\ \|f\|_2=1}}\mathcal{R}_{\mu}(f)
=
\max_{\substack{f\in F\\ \|f\|_2=1}}
\|\mu\|_{1}\,\mathbb{E}_{\mu}[p_f]\mathcal{R}(f)
\]
\[
\le
\|\mu\|_{1}
\left(\max_{\substack{f\in F\\ \|f\|_2=1}}\mathbb{E}_{\mu}[p_f]\right)
\left(\max_{\substack{f\in F\\ \|f\|_2=1}}\mathcal{R}(f)\right)
=
\lambda_k\|\mu\|_{1}
\max_{\substack{f\in F\\ \|f\|_2=1}}\mathbb{E}_{\mu}[p_f].
\]

Thus
\[
\min_{F\in\mathcal{F}_{k}}
\max_{\substack{f\in F\\ \|f\|_2=1}}
\mathcal{R}_{\mu}(f)
\le
\lambda_k\|\mu\|_{1}
\min_{F\in\mathcal{F}_{k}}
\max_{\substack{f\in F\\ \|f\|_2=1}}
\mathbb{E}_{\mu}[p_f].
\]

And
\[
\min_{\substack{F\subset\mathbb{R}^n\\ \dim(F)=k+1}}
\max_{\substack{f\in F\\ \|f\|_2=1}}
\mathcal{R}_{\mu}(f)
\le
\min_{F\in\mathcal{F}_{k+1}}
\max_{\substack{f\in F\\ \|f\|_2=1}}
\mathcal{R}_{\mu}(f),
\]
where
\[
\min_{\substack{F\subset\mathbb{R}^n\\ \dim(F)=k+1}}
\max_{\substack{f\in F\\ \|f\|_2=1}}
\mathcal{R}_{\mu}(f)
=
\lambda_k^{\mu}
\]
by the min-max theorem.
\begin{flushright}
    \qed
\end{flushright}% \subsection{Proof of Corollary \ref{cor:upper_bound_1}}
\subsection{Proof of Lemma \ref{lem:Rmu}}
\label{app:Rmu_factorization}
We simply write:
\[
\mathcal{R}_{\mu}(f)
= \frac12\,\mu^{\top}\cdot N(f)
= \|\mu\|_{1}\left(\frac{1}{\|\mu\|_{1}}\mu^{\top}\cdot p_f\right)\cdot \frac12\left(\mathbf{1}^{\top}\cdot N(f)\right)
= \|\mu\|_{1}\,\mathbb{E}_{\mu}[p_f]\,\mathcal{R}(f).
\]
\begin{flushright}
    \qed
\end{flushright}
\section{Details on Section \ref{ssec:weighted-graph-laplacian}}\label{app:Details-geometry}
We now introduce the geometric construction underlying our approach. Rather than adding an explicit advection field to a diffusion equation, we modify the diffusion geometry itself through a positive weight $\mu$. This distinction is important on graphs. In the continuous setting, a term such as $v\cdot\nabla f$ for a scalar function $f$ and $v$ a vector field, is well defined because both $v(x)$ and $\nabla f(x)$ live in the same tangent space. On a graph, however, node signals live on vertices, whereas discrete gradients naturally live on edges. Defining a vector field therefore requires additional non-canonical choices, such as an orientation. Dirichlet forms and Laplacians, in contrast, admit direct graph analogues. This motivates inducing a transport bias through a weighted diffusion geometry rather than by prescribing an explicit advective field.

We first recall the unweighted continuous case. Let $f,g:\Omega\subset\mathbb{R}^d\to\mathbb{R}$ be smooth functions vanishing on $\partial\Omega$. The standard Dirichlet form is
$\mathcal D(f,g)= \frac12\int_\Omega \nabla f(x)\cdot\nabla g(x)dx$. By integration by parts, using the boundary condition, one obtains $\mathcal D(f,g)=-\frac12\int_\Omega g(x)\Delta f(x)\,dx$. Hence, with respect to the standard $L^2(dx)$ inner product, the associated positive operator is $Lf = -\frac12\Delta f$. The corresponding gradient flow is the usual heat equation $\partial_t f=-Lf=\frac12\Delta f$. Thus, the standard Dirichlet form recovers the usual Laplacian and ordinary diffusion.
We now introduce a positive weight $\mu:\Omega\to\mathbb{R}_{>0}$ and define the weighted Dirichlet form
$\mathcal D_\mu(f,g) = \frac12\int_\Omega \mu(x)\nabla f(x)\cdot\nabla g(x)dx$. This form penalizes local variations differently depending on the position. Equivalently, $\mu$ changes the geometry in which diffusion takes place.

There are two useful ways of representing the operator associated with the same Dirichlet form. First, using the weighted inner product $\langle f,g\rangle_\mu=\int_\Omega f(x)g(x)\mu(x)\,dx$, integration by parts gives $\mathcal D_\mu(f,g)=\langle g,L_\mu f\rangle_\mu$, where
\begin{equation}
\label{eq:BE_operator_logmu}
L_\mu f=-\frac{1}{2\mu}\nabla\cdot(\mu\nabla f)=-\frac12\Delta f-\frac12\nabla\log\mu\cdot\nabla f.
\end{equation}

The second term in \eqref{eq:BE_operator_logmu} is first order and therefore has the form of a drift or advective contribution. However, it is not introduced through an independently prescribed velocity field. It is induced by the spatial variation of the weight $\mu$.

Equivalently, one may represent the same Dirichlet form using the standard
$L^2(dx)$ inner product. This gives
\begin{equation}\label{eq:L_tilde}
\widetilde L_\mu f=-\frac12\nabla\cdot(\mu\nabla f)=-\frac12\mu\Delta f-\frac12\nabla\mu\cdot\nabla f.
\end{equation}
The associated diffusion equation is
$\partial_t f=-\widetilde L_\mu f$.
This second representation gives the intuition of a diffusion with spatially varying diffusivity $\mu$. Expanding the weighted diffusion operator in \eqref{eq:L_tilde} produces the first-order term $\nabla\mu\cdot\nabla f$, which acts as a drift-like bias. It appears because the diffusion weight itself varies in space. This is the continuous mechanism that we now transfer to graphs.

In the unweighted graph case, the Dirichlet form is $\mathcal D^G(f,g)=\frac12\sum_{i,j}A_{ij}(f_i-f_j)(g_i-g_j)$. Using the symmetry of $A$, this rewrites as $\mathcal D^G(f,g)=\sum_i g_i\sum_j A_{ij}(f_i-f_j) =g^\top(D-A)f$. Thus, as in the continuous case, the operator represented by the unweighted graph Dirichlet form is the combinatorial Laplacian $L=D-A$.

\section{Experimental setup}\label{app:experimental-setup}
\subsection{Hyper-parameters}
We report in Tables \ref{tab:stablechebnet-synth} and \ref{tab:stablechebnet-proteins} the hyper-parameters used in Section \ref{seq:experiments}.
\begin{table}[h!]
\centering
\begin{tabular}{lcccc}
\toprule
\textbf{Hyper-parameter} & \textbf{Values in grid} & \textbf{Diam} & \textbf{SSSP} & \textbf{Ecc} \\
\midrule
Hidden dimension $d$        & 20, 30, 50                & 50   & 30   & 30 \\
Number of layers            & 1, 2, 3, 5, 10, 20        & 20   & 5    & 5 \\
Polynomial order $K$        & 3, 5, 10                  & 4    & 10   & 10 \\
Step size $\epsilon$        & 0.01, 0.10, 0.20, 0.30    & 0.40 & 0.30 & 0.30 \\
Dissipative force $\gamma$  & 0, 0.01, 0.50, 1          & 0.01 & 0.00 & 0.00 \\
Activation function         & tanh, relu                & relu & relu & relu \\
Learning rate               & 0.001, 0.003              & 0.003 & 0.003 & 0.003 \\
Weight decay                & $1 \times 10^{-6}$        & $1 \times 10^{-6}$ & $1 \times 10^{-6}$ & $1 \times 10^{-6}$ \\
\bottomrule
\end{tabular}
\vspace{10pt}
\caption{Hyper-parameter grid and best settings for our model on three synthetic graph-property benchmarks.}
\label{tab:stablechebnet-synth}
\end{table}

\begin{table}[h!]
\centering
\begin{tabular}{lc}
\toprule
\textbf{Hyper-parameter} & \textbf{Sweep} \\
\midrule
Hidden dim $d$           & 256, 512, 1024 \\
Polynomial order $K$     & 5, 10, 15 \\
Num of layers            & 3, 5, 7 \\
MLP layers               & 1, 2, 3 \\
Step size $\epsilon$     & [0.1, 1.0] \\
Dissipative force $\gamma$ & 0.01, 0.05, 0.1 \\
Batch size               & 512, 1024, 2048 \\
Learning rate            & 0.0005, 0.001, 0.005 \\
Optimizer                & Adam \\
Pos-enc type             & None, Laplacian, RW \\
Pos-enc dim              & 16, 32, 64 \\
\bottomrule
\end{tabular}
\vspace{10pt}
\caption{Hyper-parameter sweep ranges for our model on \texttt{ogbn-proteins}.}
\label{tab:stablechebnet-proteins}
\end{table}

\subsection{Barbell task description}

The Barbell dataset is designed to test whether a GNN can transmit information across a severe graph bottleneck. Each graph consists of two densely connected complete subgraphs, called the “bells,” joined by a narrow path or “bridge” as shown in Figure \ref{fig:barbell-graph}. Nodes inside each bell are highly connected, while the bridge is the only route through which information can pass from one side of the graph to the other. The learning task is node-level regression: each node must predict the average input feature of the nodes in the opposite bell. This makes the dataset a direct probe of long-range communication, because successful prediction requires information to cross the bridge rather than remain trapped within one dense cluster. Performance is measured using mean squared error (MSE), where high errors indicate either over-squashing, when information from the opposite bell fails to pass through the bridge, or over-smoothing, when node representations collapse and lose local distinctions.

\begin{figure}[h!]
    \centering
    \includegraphics[width=0.85\textwidth]{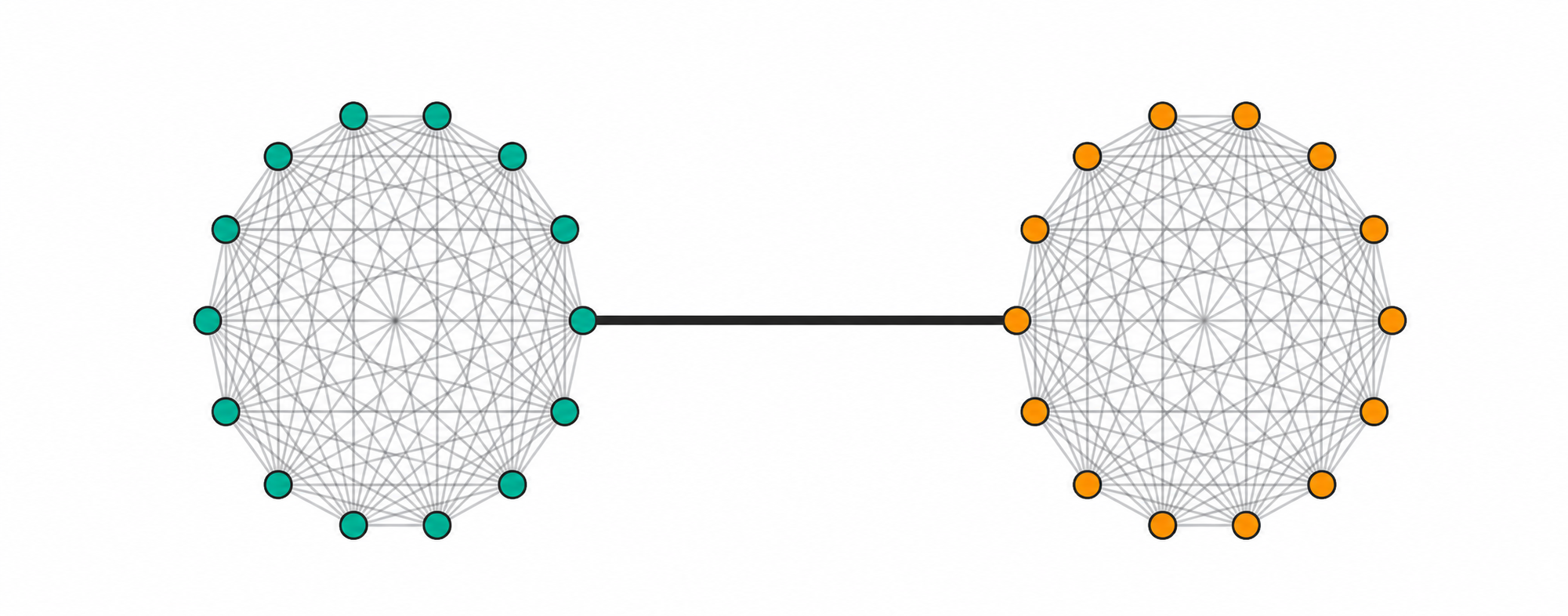}
    \caption{Illustration of a Barbell graph composed of two densely connected node clusters}
    \label{fig:barbell-graph}
\end{figure}
%%%%%%%%%%%%%%%%%%%%%%%%%%%%%%%%%%%%%%%%%%%%%%%%%%%%%%%%%%%%

\newpage
\end{document}